\documentclass{article}

\usepackage{microtype}
\usepackage{graphicx}
\usepackage{subfigure}
\usepackage{booktabs} % for professional tables

\usepackage{hyperref}

\usepackage{multirow}
\usepackage{amsmath}

\usepackage{cases}
\usepackage{enumitem}
\usepackage{xcolor}
\usepackage{subfigure}

\usepackage[accepted]{icml2021}

\icmltitlerunning{Elastic Graph Neural Networks}

\usepackage{mathtools}
\usepackage{amssymb}
\usepackage{amsthm}
\usepackage{color}
\newcommand{\cut}[1]{{}}
\newcommand{\tA}{{\tilde{\vA}}}

\newcommand{\tL}{{\tilde{\vL}}}
\newcommand{\hA}{{\hat{\vA}}}
\newcommand{\hD}{{\hat{\vD}}}

\newcommand{\tDelta}{{\tilde{\Delta}}}

\newcommand{\vf}{{\mathbf{f}}}

\newcommand{\vm}{{\mathbf{m}}}

\newcommand{\vx}{{\mathbf{x}}}

\newcommand{\vA}{{\mathbf{A}}}
\newcommand{\vB}{{\mathbf{B}}}

\newcommand{\vD}{{\mathbf{D}}}

\newcommand{\vF}{{\mathbf{F}}}

\newcommand{\vI}{{\mathbf{I}}}

\newcommand{\vL}{{\mathbf{L}}}

\newcommand{\vU}{{\mathbf{U}}}
\newcommand{\vV}{{\mathbf{V}}}

\newcommand{\vX}{{\mathbf{X}}}
\newcommand{\vY}{{\mathbf{Y}}}
\newcommand{\vZ}{{\mathbf{Z}}}

\newcommand{\cE}{{\mathcal{E}}}

\newcommand{\cG}{{\mathcal{G}}}

\newcommand{\cN}{{\mathcal{N}}}

\newcommand{\RR}{\mathbb{R}}

\newcommand{\sign}{\mathrm{sign}}
\newcommand{\vzero}{\mathbf{0}}

\newcommand{\tr}{{\mathrm{tr}}} % trace
\newcommand{\prox}{\mathbf{prox}}

\makeatletter
\let\@@span\span
\def\sp@n{\@@span\omit\advance\@multicnt\m@ne}
\makeatother

\DeclareMathOperator*{\argmin}{arg\,min}

\newcommand{\bc}{\begin{center}}
\newcommand{\ec}{\end{center}}

\newcommand{\bdm}{\begin{displaymath}}
\newcommand{\edm}{\end{displaymath}}

\newcommand{\beq}{\begin{equation}}
\newcommand{\eeq}{\end{equation}}

\newcommand{\bfl}{\begin{flushleft}}
\newcommand{\efl}{\end{flushleft}}

\newcommand{\bt}{\begin{tabbing}}
\newcommand{\et}{\end{tabbing}}

\newcommand{\beqn}{\begin{align}}
\newcommand{\eeqn}{\end{align}}

\newcommand{\beqs}{\begin{align*}} % no equation numbers
\newcommand{\eeqs}{\end{align*}}  % no equation numbers

\newtheorem{theorem}{Theorem}

\newtheorem{remark}{Remark}

\begin{document}

\twocolumn[
\icmltitle{Elastic Graph Neural Networks}



\icmlsetsymbol{equal}{*}

\begin{icmlauthorlist}
\icmlauthor{Xiaorui Liu}{cse,equal}
\icmlauthor{Wei Jin}{cse,equal}
\icmlauthor{Yao Ma}{cse}
\icmlauthor{Yaxin Li}{cse}
\icmlauthor{Hua Liu}{shandong}
\icmlauthor{Yiqi Wang}{cse}
\icmlauthor{Ming Yan}{cmse}
\icmlauthor{Jiliang Tang}{cse}
\end{icmlauthorlist}

\icmlaffiliation{cse}{Department of Computer Science and Engineering, Michigan State University, USA}
\icmlaffiliation{cmse}{Department of Computational Mathematics, Science and Engineering, Michigan State University, USA}
\icmlaffiliation{shandong}{School of Mathematics, Shandong University, China}

\icmlcorrespondingauthor{Xiaorui Liu}{xiaorui@msu.com}

\icmlkeywords{Machine Learning, ICML}

\vskip 0.3in
]



\printAffiliationsAndNotice{\icmlEqualContribution} % otherwise use the standard text.

\begin{abstract}

While many existing graph neural networks (GNNs) have been proven to perform $\ell_2$-based graph smoothing that enforces smoothness globally, in this work we aim to further enhance the local smoothness adaptivity of GNNs via $\ell_1$-based graph smoothing. As a result, we introduce a family of GNNs (Elastic GNNs) based on $\ell_1$ and $\ell_2$-based graph smoothing. In particular, we propose a novel and general message passing scheme into GNNs. This message passing algorithm is not only friendly to back-propagation training but also achieves the desired smoothing properties with a theoretical convergence guarantee. Experiments on semi-supervised learning tasks demonstrate that the proposed Elastic GNNs obtain better adaptivity on benchmark datasets and are significantly robust to graph adversarial attacks. The implementation of Elastic GNNs is available at \url{https://github.com/lxiaorui/ElasticGNN}.

\end{abstract}

\section{Introduction}

Graph neural networks (GNNs) generalize traditional deep neural networks (DNNs) from regular grids, such as image, video, and text, to irregular data such as social networks, transportation networks, and biological networks, which are typically denoted as graphs~\citep{defferrard2016convolutional, kipf2016semi}. 
One popular such generalization is the neural message passing framework~\citep{gilmer2017neural}:
\begin{align}
    \vx_u^{(k+1)} = \text{UPDATE}^{(k)} \big ( \vx_u^{(k)},  \vm_{\mathcal{N}(u)}^{(k)} \big )
\end{align}
where $\vx_u^{(k)} \in \mathbb{R}^{d}$ denotes the feature vector of node $u$ in $k$-th iteration of message passing and $\vm_{\mathcal{N}(u)}^{(k)}$ is the message aggregated from $u$'s neighborhood $\mathcal{N}(u)$. The specific architecture design has been motivated from spectral domain~\citep{kipf2016semi, defferrard2016convolutional} and spatial domain~\citep{hamilton2017inductive, velivckovic2017graph, scarselli2008graph, gilmer2017neural}. Recent study~\citep{ma2020unified} has proven that the message passing schemes in numerous popular GNNs, such as GCN, GAT, PPNP, and APPNP, 
intrinsically perform the
$\ell_2$-based graph smoothing to the graph signal, and they can be considered as solving the graph signal denoising problem: 
% \vspace{-0.05in}
\begin{align}
    \argmin_{\vF} \mathcal{L}(\vF) := \|\vF-\vX_{\text{in}} \|_F^2 + \lambda ~\text{tr}(\vF^{\top}\vL\vF),
\end{align}
where $\vX_{\text{in}} \in \mathbb{R}^{n\times d}$ is the input signal and $\vL \in \mathbb{R}^{n\times n} $ is the graph Laplacian matrix encoding the graph structure. The first term guides $\vF$ to be close to input signal $\vX_{\text{in}}$, while the second term enforces global smoothness to the filtered signal $\vF$. The resulted message passing schemes can be derived by different optimization solvers, and they typically entail the aggregation of node features from neighboring nodes, which intuitively coincides with the cluster or consistency assumption that neighboring nodes should be similar~\citep{zhu2002learning, zhou2004learning}. While existing GNNs are prominently driven by $\ell_2$-based graph smoothing, $\ell_2$-based methods enforce smoothness globally and the level of smoothness is usually shared across the whole graph. However, the level of smoothness over different regions of the graph can be different. 
For instance, node features or labels can change significantly between clusters but smoothly within the cluster~\citep{zhu2005semi}. Therefore, it is desired to enhance the local smoothness adaptivity of GNNs.

Motivated by the idea of trend filtering~\citep{kim2009ell_1, tibshirani2014adaptive, wang2016trend}, we aim to achieve the goal via $\ell_1$-based graph smoothing. Intuitively, compared with $\ell_2$-based methods, $\ell_1$-based methods penalize large values less and thus preserve discontinuity or non-smooth signal better. Theoretically, $\ell_1$-based methods tend to promote signal sparsity to trade for discontinuity~\citep{rudin1992nonlinear, tibshirani2005sparsity, sharpnack2012sparsistency}. Owning to these advantages, trend filtering~\citep{tibshirani2014adaptive} and graph trend filter~\citep{wang2016trend, varma2019vector} demonstrate that $\ell_1$-based graph smoothing can adapt to inhomogenous level of smoothness of signals and yield estimators with k-th order piecewise polynomial functions, such as piecewise constant, linear and quadratic functions, depending on the order of the graph difference operator. While $\ell_1$-based methods exhibit various appealing properties and have been extensively studied in different domains such as signal processing~\citep{elad2010sparse}, statistics and machine learning~\citep{lassobook2015}, it has rarely been investigated in the design of GNNs. In this work, we attempt to bridge this gap and enhance the local smoothnesss adaptivity of GNNs via $\ell_1$-based graph smoothing.

Incorporating $\ell_1$-based graph smoothing in the design of GNNs faces tremendous challenges. First, since the message passing schemes in GNNs can be derived from the optimization iteration of the graph signal denoising problem, a fast, efficient and scalable optimization solver is desired. Unfortunately, to solve the associated optimization problem involving $\ell_1$ norm is challenging since the objective function is composed by smooth and non-smooth components and the decision variable is further coupled by the discrete graph difference operator. Second, to integrate the derived messaging passing scheme into GNNs, it has to be composed by simple operations that are friendly to the back-propagation training of the whole GNNs. Third, it requires an appropriate normalization step to deal with diverse node degrees, which is often overlooked by existing graph total variation and graph trend filtering methods. Our attempt to address these challenges leads to a family of novel GNNs, i.e., Elastic GNNs. Our key contributions can be summarized as follows:
\vspace{-0.1in}
\begin{itemize}
\setlength\itemsep{0em}
\item We introduce $\ell_1$-based graph smoothing in the design of GNNs to further enhance the local smoothness adaptivity, for the first time;

\item We derive a novel and general message passing scheme, i.e., Elastic Message Passing (EMP), and develop a family of GNN architectures, i.e., Elastic GNNs, by integrating the proposed message passing scheme into deep neural nets; 

\item Extensive experiments demonstrate that Elastic GNNs obtain better adaptivity on various real-world datasets, and they are significantly robust to graph adversarial attacks. The study on different variants of Elastic GNNs suggests that $\ell_1$ and $\ell_2$-based graph smoothing are complementary and Elastic GNNs are more versatile.

\end{itemize}

\section{Preliminary}

We use bold upper-case letters such as $\vX$ to denote matrices and bold lower-case letters such as $\vx$ to define vectors. 
Given a matrix $\vX\in \mathbb{R}^{n\times d}$, we use $\vX_i$ to denote its $i$-th row and $\vX_{ij}$ to denote its element in $i$-th row and $j$-th column. We define the Frobenius norm, $\ell_1$ norm, and  $\ell_{21}$ norm of matrix $\vX$ as $\|\vX\|_F=\sqrt{\sum_{ij} \vX_{ij}^2}$, $\|\vX\|_1=\sum_{ij} |\vX_{ij}|$, and $\|\vX\|_{21}=\sum_i \|\vX_{i}\|_2 = \sum_i \sqrt{\sum_j \vX_{ij}^2}$, respectively. We define $\|\vX\|_2=\sigma_{\max}(\vX)$ where $\sigma_{\max}(\vX)$ is the largest singular value of $\vX$.
Given two matrices $\vX, \vY \in \mathbb{R}^{n\times d}$, we define the inner product as $\langle \vX, \vY \rangle = \text{tr}(\vX^{\top}\vY)$.

Let $\mathcal{G}=\{\mathcal{V}, \mathcal{E}\}$ be a graph with the node set $\mathcal{V}=\{v_1, \dots, v_n\}$ 
and the undirected edge set $\mathcal{E}=\{e_1, \dots, e_m \}$. We use $\cN(v_i)$ to denote the neighboring nodes of node $v_i$, including $v_i$ itself. Suppose that each node is associated with a $d$-dimensional feature vector, and the features for all nodes are denoted as $\vX_{\text{fea}} \in \mathbb{R}^{n\times d}$. 
The graph structure $\mathcal{G}$ can be represented as an adjacent matrix $\vA\in \mathbb{R}^{n\times n}$, where $\vA_{ij}=1$ when there exists an edge between nodes $v_i$ and $v_j$. The graph Laplacian matrix is defined as $\vL=\vD-\vA$, where $\vD$ is the diagonal degree matrix. 
% \xr{In GNNs, row normalized Laplacian, $\vL_{\text{rw}}=\vI-\vD^{-1}\vA$ or symmetric normalized Laplacian, $\vL_\text{{sym}}=\vI-\vD^{-1/2}\vA\vD^{-1/2}$, are often used. }
Let $\Delta \in \{-1, 0, 1\}^{m\times n}$ be the oriented incident matrix, which contains one row for each edge. If $e_{\ell}=(i,j)$, then $\Delta$ has $\ell$-th row as:
\vspace{-0.05in}
$$\Delta_{\ell} = (0, \dots, \underbrace{-1}_{i}, \dots, \underbrace{1}_{j}, \dots, 0)
\vspace{-0.05in}
$$
where the edge orientation can be arbitrary.
Note that the incident matrix and unnormalized Laplacian matrix have the equivalence $\vL=\Delta^{\top}\Delta$. Next, we briefly introduce some necessary background about the graph signal denoising perspective of GNNs and the graph trend filtering methods.

\subsection{GNNs as Graph Signal Denoising}

It is evident from recent work~\citep{ma2020unified} that many popular GNNs can be uniformly understood as graph signal denoising with Laplacian smoothing regularization. Here we briefly describe several representative examples.

\textbf{GCN.}
The message passing scheme in Graph Convolutional Networks (GCN)~\citep{kipf2016semi}, 
$$\vX_{\text{out}} = \tilde \vA \vX_{\text{in}}, $$ 
is equivalent to one gradient descent step to minimize $\tr (\vF^{\top} (\vI-\tA) \vF)$ with the initial $\vF=\vX_{\text{in}}$ and stepsize $1/2$. Here $\tA = {\hD}^{-\frac{1}{2}} \hA {\hD}^{-\frac{1}{2}}$ with $\hA=\vA+\vI$ being the adjacent matrix with self-loop,  whose degree matrix is $\hD$.

\textbf{PPNP \& APPNP.} 
The message passing scheme in PPNP and APPNP~\citep{klicpera2018predict} follow the aggregation rules
% \vspace{-0.1in}
\begin{align}
    \vX_{\text{out}} = \alpha \big (\vI - (1-\alpha) \tA \big )^{-1} \vX_{\text{in}}, \nonumber
\end{align}
\vspace{-0.1in}
and
\begin{align}
    \vX^{(k+1)} = (1-\alpha) \tA \vX^{(k)} + \alpha \vX_{\text{in}}. \nonumber
\end{align}
They are shown to be the exact solution and one gradient descent step with stepsize $\alpha/2$ for the following problem
\begin{align}
\label{equ: denoise_problem}
\min_{\vF}\|\vF-\vX_{\text{in}}\|^2_F + ({1/\alpha}-1)~ \tr (\vF^{\top}  (\vI-\tA) \vF). 
% \\
% &= \sum_{v_i \in \cV} \|\vF[i,:]-\vX[i,:] \|_2^2 + \sum_{(v_i, v_j)\in \cE} \|\vF[i,:] - \vF[j,:] \|_2^2
\end{align}
For more comprehensive illustration, please refer to~\citep{ma2020unified}. We point out that all these message passing schemes adopt $\ell_2$-based graph smoothing as the signal differences between neighboring nodes are penalized by the square of $\ell_2$ norm, e.g., 
% $\sum_{(v_i, v_j)\in \cE} \| \frac{\vF[i,:]}{\sqrt{d_i+1}} - \frac{\vF[j,:]}{\sqrt{d_j+1}} \|_2^2$ with $d_i$ 
$\sum_{(v_i, v_j)\in \cE} \| \frac{\vF_i}{\sqrt{d_i+1}} - \frac{\vF_j}{\sqrt{d_j+1}} \|_2^2$ with $d_i$ 
being the node degree of node $v_i$. The resulted message passing schemes are usually linear smoothers which smooth the input signal by their linear transformation. 
% The limitations of linear smoothers has been analyzed in the work~\citep{sadhanala2016total} and the authors 

\subsection{Graph Trend Filtering}
In the univariate case, the $k$-th order graph trend filtering (GTF) estimator~\citep{wang2016trend} is given by
\vspace{-0.1in}
\begin{align}
\label{eq:gtf}
    \argmin_{\vf \in \RR^n} = \frac{1}{2} \|\vf-\vx\|_2^2 + \lambda \|\Delta^{(k+1)} \vf \|_1
\end{align}
where $\vx\in\RR^{n}$ is the $1$-dimensional input signal of $n$ nodes and $\Delta^{(k+1)}$ is a $k$-th order graph difference operator. When $k=0$, it penalizes the absolute difference across neighboring nodes in graph $\cG$: 
% \vspace{-0.1in}
$$
\|\Delta^{(1)} \vf \|_1 = \sum_{(v_i, v_j) \in \cE} |\vf_i-\vf_j |
\vspace{-0.1in}
$$
where $\Delta^{(1)}$ is equivalent to the incident matrix $\Delta$. Generally, $k$-th order graph difference operators can be defined recursively:
\vspace{-0.1in}
\begin{align}
    \Delta^{(k+1)} = \left \{ 
    \begin{array}{cc}
    \Delta^{\top} \Delta^{(k)} = \vL^{\frac{k+1}{2}} \in  \mathbb{R}^{n\times n} & \mbox{for odd k} \\
    \Delta \Delta^{(k)} = \Delta \vL^{\frac{k}{2}} \in  \mathbb{R}^{m\times n} & \mbox{for even k}. 
    \end{array} \right. 
    \nonumber
% \vspace{-0.1in}
\end{align}
It is demonstrated that GTF can adapt to inhomogeneity in the level of smoothness of signal and tends to provide piecewise polynomials over graphs~\citep{wang2016trend}. For instance, when $k=0$, the sparsity induced by the $\ell_1$-based penalty $\|\Delta^{(1)} \vf \|_1$ implies that many of the differences $\vf_i-\vf_j$ are zeros across edges $(v_i, v_j) \in \cE$ in $\cG$. The piecewise property originates from the discontinuity of signal allowed by less aggressive $\ell_1$ penalty, with adaptively chosen knot nodes or knot edges. Note that the smoothers induced by GTF are not linear smoothers and cannot be simply represented by linear transformation of the input signal.

\section{Elastic Graph Neural Networks}

In this section, we first propose a new graph signal denoising estimator.
Then we develop an efficient optimization algorithm for solving the denoising problem and introduce a novel, general and efficient message passing scheme, i.e., Elastic Message Passing (EMP), for graph signal smoothing. Finally, the integration of the proposed message passing scheme and deep neural networks leads to Elastic GNNs.

\subsection{Elastic Graph Signal Estimator}

To combine the advantages of  $\ell_1$ and $\ell_2$-based graph smoothing, we propose the following elastic graph signal estimator:
\begin{align}
\label{eq:problem_ori}
    \argmin_{\vF \in\RR^{n \times d}} 
    \underbrace{\lambda_1 \| \Delta \vF\|_1}_{g_1(\Delta \vF)} + \underbrace{\frac{\lambda_2}{2} \tr (\vF^{\top}\vL \vF) + \frac{1}{2}\|\vF-\vX_{\text{in}} \|_F^2 }_{f(\vF)}
\end{align}
where $\vX_{\text{in}} \in \RR^{n \times d}$ is the $d$-dimensional input signal of $n$ nodes.
The first term can be written in an edge-centric way: 
$\|\Delta^{(1)} \vF \|_1 = \sum_{(v_i, v_j) \in \cE} \|\vF_i-\vF_j \|_1$, which
penalizes the absolute difference across connected nodes in graph $\cG$. Similarly, the second term penalizes the difference quadratically via $\text{tr}(\vF^{\top}\vL \vF) = \sum_{(v_i, v_j) \in \cE} \|\vF_i-\vF_j \|_2^2$.
The last term is the fidelity term which preserves the similarity with the input signal. 
The regularization coefficients $\lambda_1$ and $\lambda_2$ control the balance between  $\ell_1$ and $\ell_2$-based graph smoothing.

\begin{remark}
It is potential to consider higher-order graph differences in both the $\ell_1$-based  and $\ell_2$-based smoothers. But, in this work, we focus on the $0$-th order graph difference operator $\Delta$,
% for $\ell_1$ and $1$-th order for $\ell_2$ smoothers 
since we assume the piecewise constant prior for graph representation learning. 
\end{remark}

\textbf{Normalization.}
In existing GNNs, it is beneficial to normalize the Laplacian matrix for better numerical stability, and the normalization trick is also crucial for achieving superior performance. Therefore, for the $\ell_2$-based graph smoothing, we follow the common normalization trick in GNNs: $\tL=\vI-\tA$, where $\tA = \hD^{-\frac{1}{2}}\hA\hD^{-\frac{1}{2}}$, $\hA=\vA+\vI$ and $\hD_{ii}=d_i=\sum_j \hA_{ij}$. It leads to a degree normalized penalty 
% \vspace{-0.1in}
$$\text{tr}(\vF^{\top}\tL \vF) = \sum_{(v_i, v_j) \in \cE} \left\|\frac{\vF_i}{\sqrt{d_i+1}}-\frac{\vF_j}{\sqrt{d_j+1}} \right\|_2^2.$$
In the literature of graph total variation and graph trend filtering, the normalization step is often overlooked and the graph difference operator is directly used as in GTF~\citep{wang2016trend, varma2019vector}. To achieve better numerical stability and handle diverse node degrees in real-world graphs, we propose to normalize each column of the incident matrix by the square root of node degrees for the $\ell_1$-based graph smoothing as follows\footnote{It naturally supports read-value edge weights if the edge weights are set in the incident matrix $\Delta$.}:
% \vspace{-0.1in}
$$\tDelta = \Delta \hD^{-\frac{1}{2}}.$$
It leads to a degree normalized total variation penalty
\footnote{With the normalization, the piecewise constant prior is up to the degree scaling, i.e., sparsity in $\tDelta \vF$.}
$$\|\tDelta \vF\|_1 = \sum_{(v_i, v_j) \in \cE} \left\|\frac{\vF_i}{\sqrt{d_i+1}}-\frac{\vF_j}{\sqrt{d_j+1}} \right\|_1.$$
Note that this normalized incident matrix maintains the relation with the normalized Laplacian matrix as in the unnormalized case
% \vspace{-0.1in}
\begin{align}
\label{eq:normalized_inc}
    \tL = {\tDelta}^{\top} \tDelta
\end{align}
given that
\begin{align*}
\tL 
% & = \vI - \hat \vD^{-\frac{1}{2}} \hat \vA \hat \vD^{-\frac{1}{2}}
= \hD^{-\frac{1}{2}} (\hat \vD-\hat \vA)  \hat \vD^{-\frac{1}{2}}   
=  \hD^{-\frac{1}{2}} \vL \hat \vD^{-\frac{1}{2}} 
= \hat \vD^{-\frac{1}{2}} \Delta^{\top} \Delta \hat \vD^{-\frac{1}{2}}.
\end{align*}

With the normalization, the estimator defined in~\eqref{eq:problem_ori} becomes:
\begin{align}
\label{eq:problem_normalized_l1}
    \argmin_{\vF \in\RR^{n \times d}} 
    \underbrace{\lambda_1 \| \tDelta \vF\|_1}_{g_1(\tDelta \vF)} + \underbrace{\frac{\lambda_2}{2} \tr (\vF^{\top}\tL \vF) + \frac{1}{2}\|\vF-\vX_{\text{in}} \|_F^2 }_{f(\vF)}.
\end{align}
\textbf{Capture correlation among dimensions.} 
The node features in real-world graphs are usually multi-dimensional. Although the estimator defined in~\eqref{eq:problem_normalized_l1} is able to handle multi-dimensional data since the signal from different dimensions are separable under $\ell_1$ and $\ell_2$ norm, such estimator treats each feature dimension independently and does not exploit the potential relation between feature dimensions. However, the sparsity patterns of node difference across edges could be shared among feature dimensions. To better exploit this potential correlation, we propose to couple the multi-dimensional features by $\ell_{21}$ norm, which penalizes the summation of $\ell_2$ norm of the node difference 
$$\|\tDelta \vF\|_{21} = \sum_{(v_i, v_j) \in \cE} \left\|\frac{\vF_i}{\sqrt{d_i+1}}-\frac{\vF_j}{\sqrt{d_j+1}} \right\|_2.$$
This penalty promotes the row sparsity of $\tDelta \vF$ and enforces similar sparsity patterns among feature dimensions. In other words, if two nodes are similar, all their feature dimensions should be similar. Therefore, we define the $\ell_{21}$-based estimator as
\begin{align}
\label{eq:problem_normalized_l21}
    \argmin_{\vF \in\RR^{n \times d}} 
    \underbrace{\lambda_1 \| \tDelta \vF\|_{21}}_{g_{21}(\tDelta \vF)} + \underbrace{\frac{\lambda_2}{2} \tr (\vF^{\top}\tL \vF) + \frac{1}{2}\|\vF-\vX_{\text{in}} \|_F^2 }_{f(\vF)}
\end{align}
where $g_{21}(\cdot)=\lambda_1 \|\cdot\|_{21}$.
In the following subsections, we will use $g(\cdot)$ to represent both $g_1(\cdot)$ and $g_{21}(\cdot)$. We use $\ell_1$ to represent both $\ell_1$ and $\ell_{21}$ if not specified.

\subsection{Elastic Message Passing}

For the $\ell_2$-based graph smoother, message passing schemes can be derived from the gradient descent iterations of the graph signal denoising problem, as in the case of GCN and APPNP~\citep{ma2020unified}. However, computing the estimators defined by~\eqref{eq:problem_normalized_l1} and~\eqref{eq:problem_normalized_l21} is much more challenging because of the nonsmoothness, and the two components, i.e., $f(\vF)$ and $g(\tDelta\vF)$, are non-separable as they are coupled by the graph difference operator $\tDelta$. 
In the literature, researchers have developed optimization algorithms for the graph trend filtering problem~\eqref{eq:gtf} such as Alternating Direction Method of Multipliers (ADMM) and Newton type algorithms~\citep{wang2016trend, varma2019vector}. However, these algorithms require to solve the minimization of a non-trivial sub-problem in each single iteration, which incurs high computation complexity. Moreover, it is unclear how to make these iterations compatible with the back-propagation training of deep learning models. This motivates us to design an algorithm which is not only efficient but also friendly to back-propagation training.  To this end, we propose to solve an equivalent saddle point problem using a primal-dual algorithm with efficient computations. 

\textbf{Saddle point reformulation.} 
For a general convex function $g(\cdot)$, its conjugate function is defined as
$$g^*(\vZ) :=\sup_{\vX} \langle \vZ, \vX \rangle - g(\vX).$$ 
By using $g(\tDelta \vF) = \sup\limits_{\vZ} \langle \tDelta \vF, \vZ \rangle - g^*(\vZ)$, the problem~\eqref{eq:problem_normalized_l1} and~\eqref{eq:problem_normalized_l21} can be equivalently written as the following saddle point problem:
\begin{align}
\label{eq:saddle_point}
    \min_{\vF} \max_{\vZ} f(\vF) + \langle \tDelta \vF, \vZ \rangle - g^*(\vZ).
\end{align}
where $\vZ \in \RR^{m\times d}$.
Motivated by Proximal Alternating Predictor-Corrector (PAPC)~\citep{loris2011generalization,chen2013primal}, we propose an efficient algorithm with per iteration low computation complexity and convergence guarantee:
\begin{numcases}{}
% \label{eq:opt}
\bar \vF^{k+1} &$= \vF^{k} - \gamma \nabla f(\vF^k) - \gamma \tDelta^{\top} \vZ^k$, \label{eq:opt_step1} \\
    \vZ^{k+1} &$= \prox_{\beta g^*} (\vZ^k+ \beta \tDelta \bar \vF^{k+1})$, \label{eq:opt_step2} \\
    \vF^{k+1} &$= \vF^{k} - \gamma \nabla f(\vF^k) - \gamma \tDelta^{\top} \vZ^{k+1}$, \label{eq:opt_step3}
\end{numcases}
where $\prox_{\beta g^*}(\vX) = \argmin\limits_{\vY} \frac{1}{2}\|\vY-\vX\|_F^2 + \beta g^*(\vY)$. The stepsizes, $\gamma$ and $\beta$, will be specified later.
The first step~\eqref{eq:opt_step1} obtains a prediction of $\vF^{k+1}$, i.e., $\bar \vF^{k+1}$, by a gradient descent step on primal variable $\vF^k$. The second step~\eqref{eq:opt_step2} is a proximal dual ascent step on the dual variable $\vZ^k$ based on the predicted $\bar \vF^{k+1}$. Finally, another gradient descent step on the primal variable based on $(\vF^{k}, \vZ^{k+1})$ gives next iteration $\vF^{k+1}$~\eqref{eq:opt_step3}. 
Algorithm~\eqref{eq:opt_step1}--\eqref{eq:opt_step3} can be interpreted as a ``predict-correct'' algorithm for the saddle point problem~\eqref{eq:saddle_point}. 
Next we demonstrate how to compute the proximal operator in Eq.~\eqref{eq:opt_step2}.

\textbf{Proximal operators.}
Using the Moreau's decomposition principle~\citep{10.5555/2028633} 
$$\vX=\prox_{\beta g^*}(\vX) + \beta \prox_{\beta^{-1} g} (\vX/\beta),$$ 
we can rewrite the step~\eqref{eq:opt_step2} using the proximal operator of $g(\cdot)$, that is, 
\vspace{-0.1in}
\begin{align}
\label{eq:moreau}
    \prox_{\beta g^*}(\vX) =  
    \vX - \beta \prox_{\frac{1}{\beta} g}(\frac{1}{\beta} \vX). 
\end{align}

\begin{figure*}[!ht]
% \begin{figure*}[!Ht]
\centering
\colorbox{gray!10}{
\centering 
\begin{minipage}{1.7\columnwidth}
% \begin{minipage}{1.8\columnwidth}
% \begin{numcases}{\textbf{Message Passing~~}}
\begin{numcases}{}
\vY^{k+1}= \gamma \vX_{\text{in}} + (1-\gamma) \tA \vF^k \label{eq:mp_step0} \nonumber \\
% \vY^k = \big ( (1-\gamma)\vI - \gamma\lambda_2 \tL \big ) \vF^k + \gamma \vX_{\text{in}} \label{eq:mp_step0} \nonumber \\
% \bar\vZ^{k+1} = (\vI-\gamma \beta \tL) \vZ^k + \beta \tDelta \vY^k \\
\bar \vF^{k+1} = \vY^k -\gamma \tDelta^{\top} \vZ^k \nonumber \\
\bar \vZ^{k+1} = \vZ^k + \beta \tDelta \bar \vF^{k+1} \nonumber \\
\begin{cases}
\vZ^{k+1} = \min( |\bar \vZ^{k+1} |, \lambda_1) \cdot \sign (\bar \vZ^{k+1}) ~~~~~~~~~~~~~~~~~~\mbox{({\bf Option I:} ~}\ell_1\mbox{ norm)}  \\
\vZ^{k+1}_i =  \min(\|\bar \vZ^{k+1}_i \|_2, \lambda_1) \cdot \frac{\bar \vZ^{k+1}_i}{ \| \bar \vZ^{k+1}_i\|_2}, \forall i\in [m] \hfill ~~~~\mbox{({\bf Option II:} ~}\ell_{21} \mbox{ norm)}
\label{eq:mp_step2}
\end{cases} \nonumber \\
\vF^{k+1} = \vY^k - \gamma \tDelta^{\top} \vZ^{k+1} \nonumber \label{eq:mp_step3}
\end{numcases}
% \vspace{0.01in}
\end{minipage}
}

\caption{Elastic Message Passing (EMP). 
$\vF^0=\vX_{\text{in}}$ and $\vZ^0=\vzero^{m\times d}$.
}
\label{fig:mp}
\end{figure*}

%
% \vspace{-0.1in}
We discuss the two options for the function $g(\cdot)$ corresponding to the objectives~\eqref{eq:problem_normalized_l1} and~\eqref{eq:problem_normalized_l21}.
\begin{itemize}[leftmargin=0.15in]
    \item
    \textbf{Option I} ($\ell_{1}$ norm): $g_1(\vX) = \lambda_1 \|\vX\|_1$
    
By definition, the proximal operator of $\frac{1}{\beta} g_1(\vX)$ is
$$\prox_{\frac{1}{\beta} g_1}(\vX) = \argmin_{\vY} \frac{1}{2}\|\vY-\vX\|_F^2 + \frac{1}{\beta} \lambda_1 \|\vY\|_1,$$
which is equivalent to the soft-thresholding operator (component-wise): 
% \vspace{-0.1in}
\begin{align*}
    (S_{\frac{1}{\beta}  \lambda_1}(\vX))_{ij}=&\sign(\vX_{ij})\max(|\vX_{ij}|-\frac{1}{\beta} \lambda_1,0)\\
    =& \vX_{ij}-\sign(\vX_{ij})\min(|\vX_{ij}|, \frac{1}{\beta} \lambda_1).
\end{align*}
Therefore, using~\eqref{eq:moreau}, we have 
\begin{align}
\label{eq:proximal_l1}
    (\prox_{\beta g_1^*}(\vX))_{ij} =
     \sign(\vX_{ij})\min(|\vX_{ij}|, \lambda_1).
\end{align}
which is a \textit{component-wise projection} onto the $\ell_{\infty}$ ball of radius $\lambda_1$.

\item \textbf{Option II} ($\ell_{21}$ norm): $g_{21}(\vX)= \lambda_1 \|\vX\|_{21}$

By definition, the proximal operator of $\frac{1}{\beta} g_{21}(\vX)$ is
\begin{align*}
    \prox_{\frac{1}{\beta}  g_{21}} (\vX) = \argmin_{\vY} \frac{1}{2}\|\vY-\vX\|^2_F + \frac{1}{\beta} \lambda_1 \|\vY\|_{21} 
\end{align*}
with the $i$-th row being
$$\big (\prox_{ \frac{1}{\beta}  g_{21}} (\vX) \big )_i = 
\frac{\vX_{i}}{\|\vX_i\|_2}  \max(\|\vX_i\|_2 - \frac{1}{\beta} \lambda_1,0).$$
Similarly, using~\eqref{eq:moreau}, we have the $i$-th row of $\prox_{\beta g_{21}^*}(\vX)$ being
\begin{align}
\label{eq:proximal_l21}
& (\prox_{\beta g_{21}^*} (\vX))_i \nonumber \\ 
&= \vX_i - \beta \prox_{\frac{1}{\beta}g_{21}}(\vX_i/\beta) \nonumber  \\
&= \vX_i - \beta \frac{\vX_i/\beta}{\|\vX_i/\beta\|_2 } \max(\|\vX_i/\beta\|_2- \lambda_1/\beta, 0) \nonumber  \\
&= \vX_i - \frac{\vX_i}{\|\vX_i\|_2 } \max(\|\vX_i\|_2-\lambda_1, 0) \nonumber  \\
&= \frac{\vX_i}{\|\vX_i\|_2 } (\|\vX_i\|_2 - \max(\|\vX_i\|_2-\lambda_1, 0) ) \nonumber  \\
&= \frac{\vX_i}{\|\vX_i\|_2 }\min(\|\vX_i\|_2, \lambda_1),
\end{align}
which is a \textit{row-wise projection} on the $\ell_{2}$ ball of radius $\lambda_1$. Note that the proximal operator in the $\ell_1$ norm case treats each feature dimension independently, while in the $\ell_{21}$ norm case, it couples the multi-dimensional features, which is consistent with the motivation to exploit the correlation among feature dimensions.
\end{itemize}

The Algorithm~\eqref{eq:opt_step1}--\eqref{eq:opt_step3} and the proximal operators~\eqref{eq:proximal_l1} and~\eqref{eq:proximal_l21} enable us to derive the final message passing scheme.
Note that the computation $\vF^{k} - \gamma \nabla f(\vF^k)$ in steps~\eqref{eq:opt_step1} and~\eqref{eq:opt_step3} can be shared to save computation. Therefore, we decompose the step~\eqref{eq:opt_step1} into two steps:
% Using the gradient $\nabla f(\vF^k)=\vF^k-\vX_{\text{in}} +\lambda_2\tL\vF^k$, 
\begin{align}
\vY^k &=\vF^k-\gamma \nabla f(\vF^k) \nonumber \\
&=\big ( (1-\gamma)\vI - \gamma\lambda_2 \tL \big ) \vF^k + \gamma \vX_{\text{in}}, \label{eq:gd_step} \\
\bar \vF^{k+1} &= \vY^k -\gamma \tDelta^{\top} \vZ^k.
\end{align}
In this work, we choose $\gamma = \frac{1}{1+\lambda_2}$ and $\beta=\frac{1}{2\gamma}$. Therefore, with $\tL=\vI-\tA$, Eq.~\eqref{eq:gd_step} can be simplified as 
\begin{align}
\vY^{k+1}= \gamma \vX_{\text{in}} + (1-\gamma) \tA \vF^k.
\end{align}
Let $\bar \vZ^{k+1} := \vZ^k + \beta \tDelta \bar \vF^{k+1}$, then steps~\eqref{eq:opt_step2} and~\eqref{eq:opt_step3} become 
% \vspace{-0.1in}
\begin{align}
\vZ^{k+1} &= \prox_{\beta g^*}(\bar \vZ^{k+1}), \label{eq:proximal_derive}\\
\vF^{k+1} &= \vF^k-\gamma \nabla f(\vF^k)-\gamma\tDelta\vZ^{k+1} \nonumber \\
          &= \vY^k - \gamma \tDelta^{\top} \vZ^{k+1}.
\end{align}
Substituting the proximal operators in~\eqref{eq:proximal_derive} with ~\eqref{eq:proximal_l1} and~\eqref{eq:proximal_l21}, we obtain the complete elastic message passing scheme (EMP) as summarized in Figure~\ref{fig:mp}. 

\textbf{Interpretation of EMP.} 
EMP can be interpreted as the standard message passing (MP) ($\vY$ in Fig. 1) with extra operations (the following steps). The extra operations compute $\tDelta^\top\vZ$ to adjust the standard MP such that sparsity in $\tDelta \vF$ is promoted and some large node differences can be preserved.
EMP is general and covers some existing propagation rules as special cases as demonstrated in Remark~\ref{thm:special}.
\begin{remark}[Special cases]
\label{thm:special}
If there is only $\ell_2$-based regularization, i.e., $\lambda_1=0$, then according to the projection operator, we have $\vZ^k=\vzero^{m\times n}$. Therefore, with $\gamma=\frac{1}{1+\lambda_2}$, the proposed message passing scheme reduces to
% \vspace{-0.1in}
$$\vF^{k+1}= \frac{1}{1+\lambda_2} \vX_{\text{in}} + \frac{\lambda_2}{1+\lambda_2} \tA \vF^k.$$
If $\lambda_2=\frac{1}{\alpha}-1$, it recovers the message passing in APPNP:
% \vspace{-0.1in}
$$\vF^{k+1}= \alpha \vX_{\text{in}} + (1-\alpha) \tA \vF^k.$$
If $\lambda_2=\infty$, it recovers the simple aggregation operation in many GNNs:
$$\vF^{k+1}= \tA \vF^k.$$
\end{remark}

\textbf{Computation Complexity.} 
EMP is efficient and composed by simple operations. The major computation cost comes from four sparse matrix multiplications, include $\tA \vF^k, \tDelta^{\top} \vZ^k, \tDelta \bar\vF^{k+1}$ and $\tDelta^{\top} \vZ^{k+1}$. The computation complexity is in the order $O (md)$ where $m$ is the number of edges in graph $\cG$ and $d$ is the feature dimension of input signal $\vX_{\text{in}}$. 
Other operations are simple matrix additions and projection.

The convergence of EMP and the parameter settings are justified by Theorem~\ref{thm:converge}, with a proof deferred to Appendix~\ref{supp:convergence}.
\begin{theorem}[Convergence]
\label{thm:converge}
Under the stepsize setting $\gamma < \frac{2}{1+\lambda_2\|\tL\|_2}$ and $\beta \leq \frac{4}{3\gamma \|\tDelta \tDelta^{\top}\|_2}$, the elastic message passing scheme (EMP) in Figure~\ref{fig:mp} converges to the optimal solution of the elastic graph signal estimator defined in~\eqref{eq:problem_normalized_l1} (Option I) or ~\eqref{eq:problem_normalized_l21} (Option II). It is sufficient to choose any $\gamma < \frac{2}{1+2\lambda_2}$ and $\beta \leq \frac{2}{3\gamma}$ since $\|\tL\|_2 = \|\tDelta^{\top} \tDelta\|_2  = \|\tDelta \tDelta^{\top}\|_2  \leq 2$.
\end{theorem}

\subsection{Elastic GNNs}

Incorporating the elastic message passing scheme from the elastic graph signal estimator~\eqref{eq:problem_normalized_l1} and~\eqref{eq:problem_normalized_l21} into deep neural networks, we introduce a family of GNNs, namely Elastic GNNs. 
In this work, we follow the decoupled way as proposed in APPNP~\citep{klicpera2018predict}, where we first make predictions from node features and aggregate the prediction through the proposed EMP:
\begin{align}
\vY_{\text{pre}} = \textbf{EMP}~\big (h_{\theta}(\vX_{\text{fea}}), K, \lambda_1, \lambda_2\big).
\end{align}
$\vX_{\text{fea}} \in \RR^{n\times d}$ denotes the node features, $h_{\theta}(\cdot)$ is any machine learning model, such as multilayer perceptrons (MLPs), $\theta$ is the learnable parameters in the model, and $K$ is the number of message passing steps. The training objective is the cross entropy loss defined by the final prediction $\vY_{\text{pre}}$ and labels for training data. Elastic GNNs also have the following nice properties:
% \vspace{-0.05in}
\begin{itemize}
\setlength\itemsep{0em}
    \item In addition to the backbone neural network model, Elastic GNNs only require to set up three hyperparameters including two coefficients $\lambda_1, \lambda_2$ and the propagation step $K$, but they do not introduce any learnable parameters. Therefore, it reduces the risk of overfitting. 
    
    \item The hyperparameters $\lambda_1$ and $\lambda_2$ provide better smoothness adaptivity to Elastic GNNs depending on the smoothness properties of the graph data.
    
    \item The message passing scheme only entails simple and efficient operations, which makes it friendly to the efficient and end-to-end back-propagation training of the whole GNN model.
\end{itemize}

\section{Experiment}

In this section, we conduct experiments to validate the effectiveness of the proposed Elastic GNNs. We first introduce the experimental settings. Then we assess the performance of Elastic GNNs and investigate the benefits of introducing $\ell_1$-based graph smoothing into GNNs with semi-supervised learning tasks under normal and adversarial settings. In the ablation study, we validate the local adaptive smoothness, sparsity pattern, and convergence of EMP.

\subsection{Experimental Settings}

\textbf{Datasets.}  We conduct experiments on 8 real-world datasets including three citation graphs, i.e., Cora, Citeseer, Pubmed~\citep{sen2008collective}, two co-authorship graphs, i.e., Coauthor CS and Coauthor Physics~\citep{shchur2018pitfalls}, two co-purchase graphs, i.e., Amazon Computers and Amazon Photo~\citep{shchur2018pitfalls}, and one blog graph, i.e., Polblogs~\citep{adamic2005political}. In Polblogs graph, node features are not available so we set the feature matrix to be a $n\times n$ identity matrix. 
% \xr{add ogbn-arxiv}

\textbf{Baselines.} We compare the proposed Elastic GNNs with representative GNNs including GCN~\citep{kipf2016semi}, GAT~\citep{velivckovic2017graph}, ChebNet~\citep{defferrard2016convolutional}, GraphSAGE~\citep{hamilton2017inductive}, APPNP~\citep{klicpera2018predict} and SGC~\citep{wu2019simplifying}. 
For all models, we use $2$ layer neural networks with $64$ hidden units. 

\textbf{Parameter settings.} For each experiment, we report the average performance and the standard variance of 10 runs. 
For all methods, hyperparameters are tuned from the following search space: 
1) learning rate: $\{0.05, 0.01, 0.005\}$; 
2) weight decay: \{5e-4, 5e-5, 5e-6\}; 3) dropout rate: \{0.5, 0.8\}. 
For APPNP, the propagation step $K$ is tuned from $\{5, 10\}$ and the parameter $\alpha$ is tuned from $\{0, 0.1, 0.2, 0.3, 0.5, 0.8, 1.0\}$. For Elastic GNNs, the propagation step $K$ is tuned from $\{5, 10\}$ and parameters $\lambda_1$ and  $\lambda_2$ are tuned from $\{0, 3, 6, 9\}$. As suggested by Theorem 1, we set $\gamma=\frac{1}{1+\lambda_2}$ and $\beta=\frac{1}{2\gamma}$ in the proposed elastic message passing scheme.
Adam optimizer~\citep{kingma2014adam} is used in all experiments. 
% The implementation of Elastic GNNs is publicly available \footnote{\url{https://github.com/lxiaorui/ElasticGNN}}.

\subsection{Performance on Benchmark Datasets}
\label{sec:benchmark}

On commonly used datasets including Cora, CiteSeer, PubMed, Coauthor CS, Coauthor Physics, Amazon Computers and Amazon Photo, we compare the performance of the proposed Elastic GNN ($\ell_{21}+\ell_2$) with representative GNN baselines on the semi-supervised learning task. The detail statistics of these datasets and data splits are summarized in Table~\ref{tab:dataset} in Appendix~\ref{sec:data}. The classification accuracy are showed in Table~\ref{tab:result_normal}. From these results, we can make the following observations:
% \vspace{-0.1in}
\begin{itemize}
\setlength\itemsep{0em}
    \item Elastic GNN outperforms GCN, GAT, ChebNet, GraphSAGE and SGC by significant margins on all datasets. For instance, Elastic GNN improves over GCN by $3.1\%$, $2.0\%$ and $1.8\%$ on Cora, CiteSeer and PubMed datasets.
    The improvement comes from the global and local smoothness adaptivity of Elastic GNN.
    
    \item Elastic GNN ($\ell_{21}+\ell_2$) consistently achieves higher performance than APPNP on all datasets. Essentially, Elastic GNN covers APPNP as a special case when there is only $
    \ell_2$ regularization, i.e., $\lambda_1=0$. Beyond the $\ell_2$-based graph smoothing, the $\ell_{21}$-based graph smoothing further enhances the local smoothness adaptivity. This comparison verifies the benefits of introducing $\ell_{21}$-based graph smoothing in GNNs.
\end{itemize}

% \begin{table*}[!ht]
% \caption{Classification accuracy ($\%$) on benchmark datasets with random splits.}
% \vskip 0.2em
% \label{tab:result_normal}
% \centering
% \resizebox{1.5\columnwidth}{!}{
% \begin{tabular}{lcccccccc}
% \toprule
% \textbf{Model} & \textbf{Cora} & \textbf{CiteSeer} & \textbf{PubMed} & \textbf{CS} & \textbf{Physics} & \textbf{Computers} & \textbf{Photo}   \\
% \midrule
% ChebNet & $76.3\pm1.5$ & $67.4\pm1.5$ & $75.0\pm2.0$ &$91.8\pm0.4$ & OOM & $\textbf{81.0}\pm\textbf{2.0}$ &$90.4\pm1.0$ \\
% GCN     & $79.6\pm1.1$ &$68.9\pm1.2$  & $77.6\pm2.3$ &$91.6\pm0.6$ &$93.3\pm0.8$ & $79.8\pm1.6$ &$90.3\pm1.2$ \\
% GAT     & $80.1\pm1.2$ &$68.9\pm1.8$  & $77.6\pm2.2$ &$91.1\pm0.5$ &$93.3\pm0.7$ & $79.3\pm2.4$ &$89.6\pm1.6$ \\
% SGC     &$80.2\pm1.5$   &$68.9\pm1.3$ &$75.5\pm2.9$ &$90.1\pm1.3$ &$93.1\pm0.6$ & $73.0\pm2.0$ & $83.5\pm2.9$  \\
% APPNP   &$82.2\pm1.3$   &$70.4\pm1.2$ & $79.5\pm2.2$ &$91.9\pm0.4$ & $93.7\pm0.7$ & $80.1\pm2.1$  & $91.4\pm1.1$  \\
% GraphSAGE &$79.0\pm1.1$ &$67.5\pm2.0$ & $77.6\pm2.0$ &$91.7\pm0.5$ &$92.5\pm0.8$ & $80.7\pm1.7$ & $90.9\pm1.0$ \\
% \textbf{ElasticGNN} & $\textbf{82.8}\pm\textbf{1.5}$ & $\textbf{70.5}\pm\textbf{1.4}$ & $\textbf{80.2}\pm\textbf{1.9}$ & $\textbf{92.5}\pm\textbf{0.3}$  & $\textbf{94.2}\pm\textbf{0.5}$  &  $80.8\pm1.8$  &$\textbf{91.6}\pm\textbf{1.3}$  \\
% \bottomrule
% \end{tabular}
% }
% \end{table*}

\begin{table*}[!ht]
\caption{Classification accuracy ($\%$) on benchmark datasets with $10$ times random data splits.}
% \vskip 0.2em
\vskip 1em
\label{tab:result_normal}
\centering
\resizebox{1.5\columnwidth}{!}{
\begin{tabular}{lcccccccc}
\toprule
\textbf{Model} & \textbf{Cora} & \textbf{CiteSeer} & \textbf{PubMed} & \textbf{CS} & \textbf{Physics} & \textbf{Computers} & \textbf{Photo}   \\
\midrule
ChebNet & $76.3\pm1.5$ & $67.4\pm1.5$ & $75.0\pm2.0$ &$91.8\pm0.4$ & OOM & $\textbf{81.0}\pm\textbf{2.0}$ &$90.4\pm1.0$ \\
GCN     & $79.6\pm1.1$ &$68.9\pm1.2$  & $77.6\pm2.3$ &$91.6\pm0.6$ &$93.3\pm0.8$ & $79.8\pm1.6$ &$90.3\pm1.2$ \\
GAT     & $80.1\pm1.2$ &$68.9\pm1.8$  & $77.6\pm2.2$ &$91.1\pm0.5$ &$93.3\pm0.7$ & $79.3\pm2.4$ &$89.6\pm1.6$ \\
SGC     &$80.2\pm1.5$   &$68.9\pm1.3$ &$75.5\pm2.9$ &$90.1\pm1.3$ &$93.1\pm0.6$ & $73.0\pm2.0$ & $83.5\pm2.9$  \\

APPNP   &$82.2\pm1.3$   &$70.4\pm1.2$ & $78.9\pm2.2$ &$\textbf{92.5}\pm\textbf{0.3}$ & $93.7\pm0.7$ & $80.1\pm2.1$  & $90.8\pm1.3$  \\
GraphSAGE &$79.0\pm1.1$ &$67.5\pm2.0$ & $77.6\pm2.0$ &$91.7\pm0.5$ &$92.5\pm0.8$ & $80.7\pm1.7$ & $90.9\pm1.0$ \\
\textbf{ElasticGNN} & $\textbf{82.7}\pm\textbf{1.0}$ & $\textbf{70.9}\pm\textbf{1.4}$ & $\textbf{79.4}\pm\textbf{1.8}$ & $\textbf{92.5}\pm\textbf{0.3}$  & $\textbf{94.2}\pm\textbf{0.5}$  &  $80.7\pm1.8$  &$\textbf{91.3}\pm\textbf{1.3}$  \\
\bottomrule
\end{tabular}
}
\end{table*}

\vspace{0.1in}

\begin{table*}[!ht] 
\caption{Classification accuracy ($\%$) under different perturbation rates of adversarial graph attack. }
\vskip 1em
% \vskip 0.2em
\label{tab:metattack}
\centering
\resizebox{1.6\columnwidth}{!}{

\begin{tabular}{c|c|cc|ccccccc}
\toprule
\multirow{2}{*}{Dataset} &  \multirow{2}{*}{Ptb Rate} & \multicolumn{2}{c|}{Basic GNN} & \multicolumn{5}{c}{Elastic GNN}    \\
% &                     & GCN               & \multicolumn{1}{c|}{GAT}                   &  $\ell_2$ (APPNP)             & $\ell_1$ & $\ell_{21}$   & $\ell_{1}+\ell_2$   & $\ell_{21}+\ell_2$  \\
&                     & GCN               & \multicolumn{1}{c|}{GAT}                   &  $\ell_2$  & $\ell_1$ & $\ell_{21}$   & $\ell_{1}+\ell_2$   & $\ell_{21}+\ell_2$  \\
\midrule
\multirow{6}{*}{Cora}     
& 0\%            & 83.5$\pm$0.4          & 84.0$\pm$0.7          & \textbf{85.8$\pm$0.4}  & 85.1$\pm$0.5          & 85.3$\pm$0.4        & \textbf{85.8}$\pm$\textbf{0.4}   & \textbf{85.8$\pm$0.4}   \\
& 5\%            & 76.6$\pm$0.8          & 80.4$\pm$0.7          & 81.0$\pm$1.0           & \textbf{82.3$\pm$1.1} & 81.6$\pm$1.1  & 81.9$\pm$1.4   & 82.2$\pm$0.9    \\
& 10\%           & 70.4$\pm$1.3          & 75.6$\pm$0.6          & 76.3$\pm$1.5           & 76.2$\pm$1.4 & 77.9$\pm$0.9   & 78.2$\pm$1.6  & \textbf{78.8$\pm$1.7}    \\
& 15\%           & 65.1$\pm$0.7          & 69.8$\pm$1.3          & 72.2$\pm$0.9           & 73.3$\pm$1.3 & 75.7$\pm$1.2  & 76.9$\pm$0.9   & \textbf{77.2$\pm$1.6}    \\
& 20\%           & 60.0$\pm$2.7          & 59.9$\pm$0.6          & 67.7$\pm$0.7           & 63.7$\pm$0.9 & 70.3$\pm$1.1   & 67.2$\pm$5.3  & \textbf{70.5$\pm$1.3}    \\
                  
\midrule 
\multirow{6}{*}{Citeseer} 
& 0\%            & 72.0$\pm$0.6          & 73.3$\pm$0.8          & 73.6$\pm$0.9            & 73.2$\pm$0.6 & 73.2$\pm$0.5   & 73.6$\pm$0.6 & \textbf{73.8$\pm$0.6}    \\
& 5\%            & 70.9$\pm$0.6          & 72.9$\pm$0.8          & 72.8$\pm$0.5            & 72.8$\pm$0.5 & 72.8$\pm$0.5  & \textbf{73.3$\pm$0.6}  & 72.9$\pm$0.5    \\
& 10\%           & 67.6$\pm$0.9          & 70.6$\pm$0.5          & 70.2$\pm$0.6            & 70.8$\pm$0.6 & 70.7$\pm$1.2  & 72.4$\pm$0.9  & \textbf{72.6$\pm$0.4}    \\
& 15\%           & 64.5$\pm$1.1          & 69.0$\pm$1.1          & 70.2$\pm$0.6            & 68.1$\pm$1.4 & 68.2$\pm$1.1  & 71.3$\pm$1.5  & \textbf{71.9$\pm$0.7}    \\
& 20\%           & 62.0$\pm$3.5          & 61.0$\pm$1.5          & \textbf{64.9$\pm$1.0}            & 64.7$\pm$0.8 & 64.7$\pm$0.8   & 64.7$\pm$0.8 & 64.7$\pm$0.8             \\
                       
\midrule 
\multirow{6}{*}{Polblogs} 
& 0\%            & 95.7$\pm$0.4          & 95.4$\pm$0.2     & 95.4$\pm$0.2       & \textbf{95.8$\pm$0.3} & \textbf{95.8$\pm$0.3}  & \textbf{95.8$\pm$0.3}   & \textbf{95.8$\pm$0.3}    \\
& 5\%            & 73.1$\pm$0.8          & 83.7$\pm$1.5          & 82.8$\pm$0.3            & 78.7$\pm$0.6 & 78.7$\pm$0.7   & 82.8$\pm$0.4 & \textbf{83.0$\pm$0.3}    \\
& 10\%           & 70.7$\pm$1.1          & 76.3$\pm$0.9          & 73.7$\pm$0.3           & 75.2$\pm$0.4 & 75.3$\pm$0.7    & 81.5$\pm$0.2 & \textbf{81.6$\pm$0.3}    \\
& 15\%           & 65.0$\pm$1.9          & 68.8$\pm$1.1          & 68.9$\pm$0.9           & 72.1$\pm$0.9 & 71.5$\pm$1.1   & 77.8$\pm$0.9  & \textbf{78.7$\pm$0.5}    \\
& 20\%           & 51.3$\pm$1.2          & 51.5$\pm$1.6          & 65.5$\pm$0.7           & 68.1$\pm$0.6 & 68.7$\pm$0.7   & 77.4$\pm$0.2  & \textbf{77.5$\pm$0.2}    \\
         
\midrule 
\multirow{6}{*}{Pubmed}   
& 0\%            & 87.2$\pm$0.1          & 83.7$\pm$0.4          & \textbf{88.1$\pm$0.1}           & 86.7$\pm$0.1 & 87.3$\pm$0.1  & \textbf{88.1$\pm$0.1}   & \textbf{88.1$\pm$0.1}           \\
& 5\%            & 83.1$\pm$0.1          & 78.0$\pm$0.4          & \textbf{87.1$\pm$0.2}           & 86.2$\pm$0.1 & 87.0$\pm$0.1   & \textbf{87.1$\pm$0.2}  & \textbf{87.1$\pm$0.2}           \\
& 10\%           & 81.2$\pm$0.1          & 74.9$\pm$0.4          & 86.6$\pm$0.1           & 86.0$\pm$0.2 & 86.9$\pm$0.2   & 86.3$\pm$0.1  & \textbf{87.0$\pm$0.1}    \\
& 15\%           & 78.7$\pm$0.1          & 71.1$\pm$0.5          & 85.7$\pm$0.2           & 85.4$\pm$0.2 & 86.4$\pm$0.2    & 85.5$\pm$0.1 & \textbf{86.4$\pm$0.2}     \\
& 20\%           & 77.4$\pm$0.2          & 68.2$\pm$1.0          & 85.8$\pm$0.1           & 85.4$\pm$0.1 & 86.4$\pm$0.1   & 85.4$\pm$0.1  & \textbf{86.4$\pm$0.1}     \\
\bottomrule
\end{tabular}
}
\end{table*}

\subsection{Robustness Under Adversarial Attack}
\label{sec:attack}
Locally adaptive smoothness makes Elastic GNNs more robust to adversarial attack on graph structure. This is because the attack tends to connect nodes with different labels, which fuzzes the cluster structure in the graph. But EMP can tolerate large node differences along these wrong edges, and maintain the smoothness along correct edges.

To validate this, we evaluate the performance of Elastic GNNs under untargeted adversarial graph attack, which tries to degrade GNN models' overall performance by deliberately modifying the graph structure. We use the MetaAttack~\citep{mettack} implemented in DeepRobust~\citep{li2020deeprobust}\footnote{\url{https://github.com/DSE-MSU/DeepRobust}}, a PyTorch library for adversarial attacks and defenses, to generate the adversarially attacked graphs based on four datasets including Cora, CiteSeer, Polblogs and PubMed. We randomly split $10\%/10\%/80\%$ of nodes for training, validation and test.
The detailed data statistics are summarized in Table~\ref{tab:attack_dataset} in Appendix~\ref{sec:data}. Note that following the works~\cite{nettack, mettack, all-you-need-is-low-rank, jin2020graph}, we only consider the largest connected component (LCC) in the adversarial graphs. 
Therefore, the results in Table~\ref{tab:metattack} are not directly comparable with the results in Table~\ref{tab:result_normal}. 
We focus on investigating the robustness introduced by $\ell_1$-based graph smoothing but not on adversarial defense so we don't compare with defense strategies. Existing defense strategies can be applied on Elastic GNNs to further improve the robustness against attacks. 

\textbf{Variants of Elastic GNNs.}
To make a deeper investigation of Elastic GNNs, we consider the following variants:
% (1) $\ell_2$; 
% (2) $\ell_{1}$; 
% (3) $\ell_{21}$; 
% (4) $\ell_{1}+\ell_2$; 
% (5) $\ell_{21}+\ell_2$. 
(1) $\ell_2 ~(\lambda_1=0)$; 
(2) $\ell_{1} ~(\lambda_2=0, \text{Option I})$; 
(3) $\ell_{21} ~(\lambda_2=0, \text{Option II})$; 
(4) $\ell_{1}+\ell_2 ~(\text{Option I})$; 
(5) $\ell_{21}+\ell_2 ~(\text{Option II})$. 
To save computation, we fix the learning rate as $0.01$, weight decay as $0.0005$, dropout rate as $0.5$ and $K=10$ since this setting works well for the chosen datasets and models. Only 
% $\alpha$, 
$\lambda_1$ and $\lambda_2$ are tuned.
The classification accuracy under different perturbation rates ranging from $0\%$ to $20\%$ is summarized in Table~\ref{tab:metattack}. From the results, we can make the following observations:
% \vspace{-0.1in}
\begin{itemize}
\setlength\itemsep{0em}

    \item All variants of Elastic GNNs outperforms GCN and GAT by significant margins under all perturbation rates. For instance, when the pertubation rate is $15\%$, Elastic GNN ($\ell_{21}+\ell_2$) improves over GCN by $12.1\%$, $7.4\%$, $13.7\%$ and $7.7\%$ on the four datasets being considered. This is because Elastic GNN can adapt to the change of smoothness while GCN and GAT can not adapt well when the perturbation rate increases.
    
    \item $\ell_{21}$ outperforms $\ell_{1}$ in most cases, and $\ell_{21}+\ell_2$  outperforms $\ell_{1}+\ell_2$ in almost all cases. It demonstrates the benefits of exploiting the correlation between feature channels by coupling multi-dimensional features via $\ell_{21}$ norm.
    
    \item $\ell_{21}$ outperforms $\ell_{2}$ in most cases, which suggests the benefits of local smoothness adaptivity. When $\ell_{21}$ and $\ell_2$ is combined, the Elastic GNN ($\ell_{21}+\ell_2$) achieves significantly better performance than solely $\ell_{2}$, $\ell_{21}$ or $\ell_{1}$ variant in almost all cases. It suggests that $\ell_{1}$ and $\ell_{2}$-based graph smoothing are complementary to each other, and combining them provides significant better robustness against adversarial graph attacks. 
\end{itemize}

\subsection{Ablation Study}
We provide ablation study to further investigate the adaptive smoothness, sparsity pattern, and convergence of EMP in Elastic GNN, based on three datasets including Cora, CiteSeer and PubMed. In this section, we fix $\lambda_1=3, \lambda_2=3$ for Elastic GNN, and $\alpha=0.1$ for APPNP. We fix learning rate as $0.01$, weight decay as $0.0005$ and dropout rate as $0.5$ since this setting works well for both methods.

\textbf{Adaptive smoothness.}
It is expected that $\ell_1$-based smoothing enhances local smoothness adaptivity by increasing the smoothness along correct edges (connecting nodes with same labels) while lowering smoothness along wrong edges (connecting nodes with different labels). To validate this, we compute the average adjacent node differences (based on node features in the last layer) along wrong and correct edges separately, and use the ratio between these two averages to measure the smoothness adaptivity. The results are summarized in Table~\ref{tab:node_diff_ratio}. It is clearly observed that for all datasets, the ratio for ElasticGNN is significantly higher than $\ell_2$ based method such as APPNP, which validates its better local smoothness adaptivity.

\textbf{Sparsity pattern.} 
To validate the piecewise constant property enforced by EMP, we also investigate the sparsity pattern in the adjacent node differences, i.e., $\tDelta \vF$, based on node features in the last layer. Node difference along edge $e_i$ is defined as sparse if $\|(\tDelta \vF)_i \|_2 < 0.1$. The sparsity ratios for $\ell_2$-based method such as APPNP and $\ell_1$-based method such as Elastic GNN are summarized in Table~\ref{tab:sparsity_ratio}. It can be observed that in Elastic GNN, a significant portion of $\tDelta \vF$ are sparse for all datasets. While in APPNP, this portion is much smaller. This sparsity pattern validates the piecewise constant prior as designed.

% \vspace{-0.1in}
\begin{table}[!ht]
\caption{Ratio between average node differences along wrong and correct edges.}
\vskip 0.3em
% \vskip 1em
\label{tab:node_diff_ratio}
\centering
\resizebox{0.8\columnwidth}{!}{
\begin{tabular}{c|ccc}
\toprule
{\textbf{Model}} & {\textbf{Cora}} & {\textbf{CiteSeer}} & {\textbf{PubMed}} \\
\midrule
$\ell_2$ (APPNP)     & $1.57$ & $1.35$  & $1.43$ \\
$\ell_{21}$+$\ell_2$ (ElasticGNN)  & $2.03$ & $1.94$  & $1.79$ \\
\bottomrule
\end{tabular}
}
\end{table}

\vspace{-0.1in}
\begin{table}[!ht]
\caption{Sparsity ratio (i.e., $\|(\tDelta \vF)_i\|_2 < 0.1$) in node differences $\tDelta \vF$.}
\vskip 0.3em
% \vskip 1em
\label{tab:sparsity_ratio}
\centering
\resizebox{0.8\columnwidth}{!}{
\begin{tabular}{c|ccc}
\toprule
{\textbf{Model}} & {\textbf{Cora}} & {\textbf{CiteSeer}} & {\textbf{PubMed}} \\
\midrule
$\ell_2$ (APPNP)     & $2\%$ & $16\%$  & $11\%$ \\
$\ell_{21}$+$\ell_2$ (ElasticGNN) & $37\%$ & $74\%$  & $42\%$ \\
\bottomrule
\end{tabular}
}
\end{table}

% \vspace{0.3in}
\textbf{Convergence of EMP.}
We provide two additional experiments to demonstrate the impact of propagation step $K$ on classification performance and the convergence of message passing scheme. Figure~\ref{fig:changeK_acc} shows that the increase of classification accuracy when the propagation step $K$ increases. It verifies the effectiveness of EMP in improving graph representation learning. It also shows that a small number of propagation step can achieve very good performance, and therefore the computation cost for EMP can be small. 
Figure~\ref{fig:changeK_converge} shows the decreasing of the objective value defined in Eq.~\eqref{eq:problem_normalized_l21} during the forward message passing process, and it verifies the convergence of the proposed EMP as suggested by Theorem~\ref{thm:converge}.

\begin{figure}[h]
\centering
\includegraphics[width = 0.9\columnwidth]{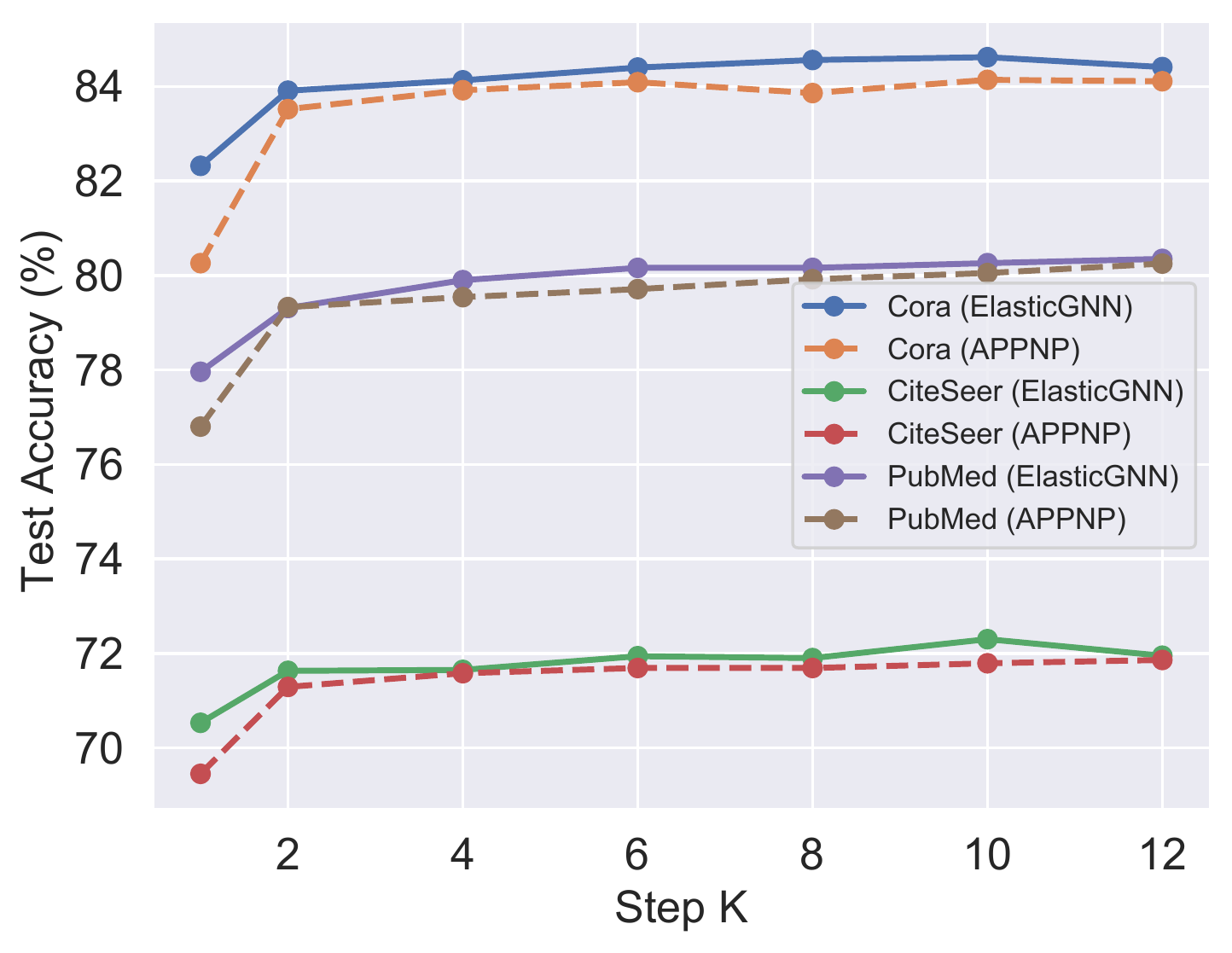}
\vspace{-0.15in}
\caption{Classification accuracy under different propagation steps.}
\label{fig:changeK_acc}
\end{figure}
\begin{figure}[h]
\centering
\includegraphics[width = 0.9\columnwidth]{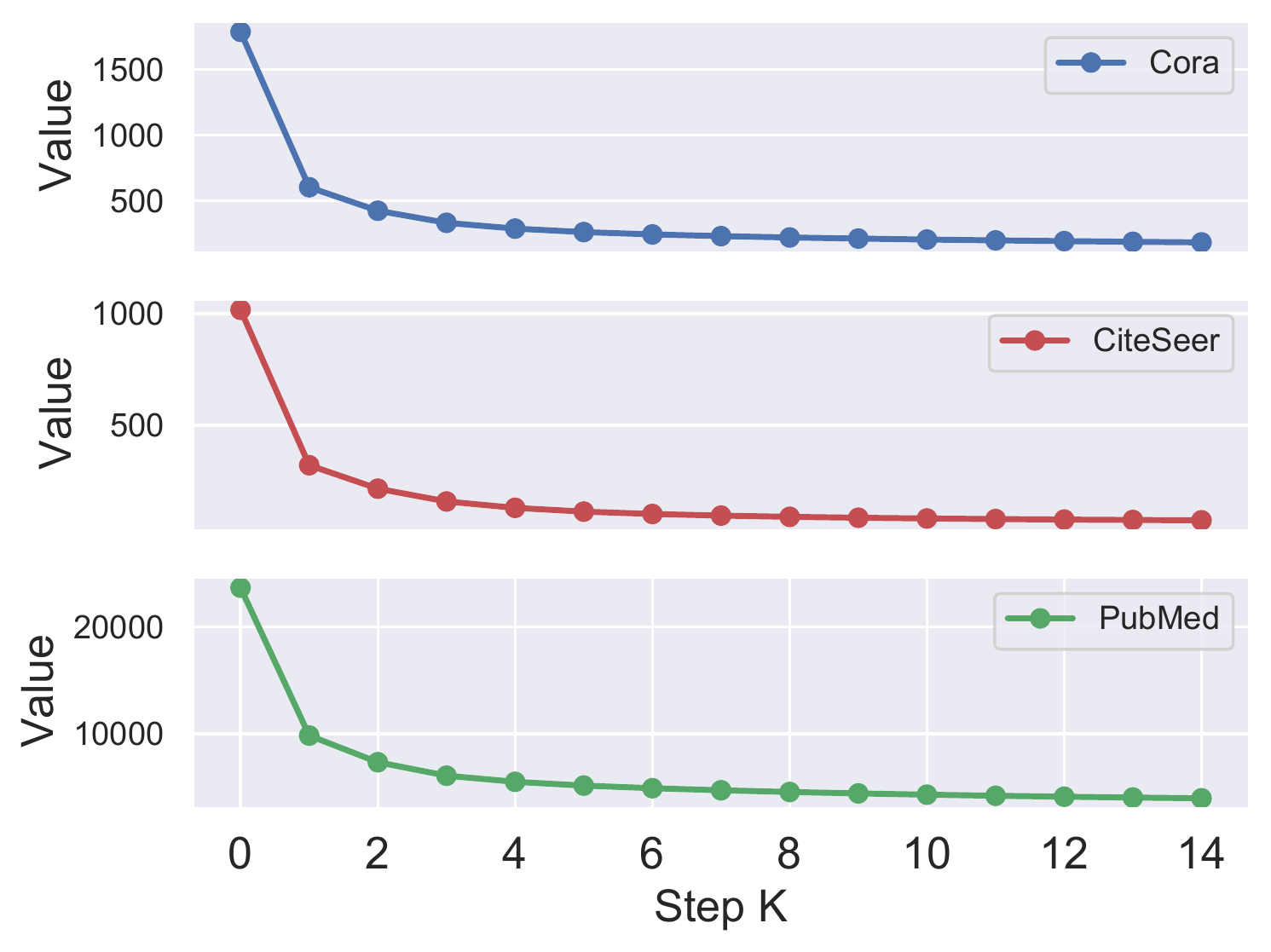}
\vspace{-0.15in}
% \caption{ Convergence of the objective value for the elastic graph denoising problem during message passing.}
\caption{ Convergence of the objective value for the problem in Eq.~\eqref{eq:problem_normalized_l21} during message passing.}
\label{fig:changeK_converge}
\end{figure}

\section{Related Work}

The design of GNN architectures can be majorly motivated in spectral domain~\citep{kipf2016semi, defferrard2016convolutional} and spatial domain~\citep{hamilton2017inductive, velivckovic2017graph, scarselli2008graph, gilmer2017neural}. The message passing scheme~\citep{gilmer2017neural, ma2020deep} for feature aggregation is one central component of GNNs. Recent works have proven that the message passing in GNNs can be regarded as low-pass graph filters~\citep{nt2019revisiting, zhao2019pairnorm}. Generally, it is recently proved that message passing in many GNNs can be unified in the graph signal denosing framework~\citep{ma2020unified, pan2020_unified, zhu2021interpreting, chen2020graph}. We point out that they intrinsically perform $\ell_2$-based graph smoothing and typically can be represented as linear smoothers. 

$\ell_1$-based graph signal denoising has been explored in graph trend filtering~\citep{wang2016trend, varma2019vector} which tends to provide estimators with $k$-th order piecewise polynomials over graphs. Graph total variation has also been utilized in semi-supervised learning~\citep{nie2011unsupervised, jung2016semi, semi_tv_min, aviles2019labelled}, spectral clustering~\citep{buhler2009spectral, bresson2013multiclass} and graph cut problems~\citep{szlam2010total, bresson2013adaptive}. However, it is unclear whether these algorithms can be used to design GNNs. To the best of our knowledge, we make first such investigation in this work.

\section{Conclusion}

In this work, we propose to enhance the smoothness adaptivity of GNNs via $\ell_1$ and $\ell_2$-based graph smoothing. Through the proposed elastic graph signal estimator, we derive a novel, efficient and general message passing scheme, i.e., elastic message passing (EMP). Integrating the proposed message passing scheme and deep neural networks leads to a family of GNNs, i.e., Elastic GNNs. Extensitve experiments on benchmark datasets and adversarially attacked graphs demonstrate the benefits of introducing $\ell_1$-based graph smoothing in the design of GNNs. The empirical study suggests that $\ell_1$ and $\ell_2$-based graph smoothing is complementary to each other, and the proposed Elastic GNNs has better smoothnesss adaptivity owning to the integration of $\ell_1$ and $\ell_2$-based graph smoothing. We hope the proposed elastic message passing scheme can inspire more powerful GNN architecture design in the future.

\section*{Acknowledgements}

This research is supported by the National Science Foundation (NSF) under grant numbers CNS-1815636, IIS-1928278, IIS-1714741, IIS-1845081, IIS-1907704, IIS-1955285, and Army Research Office (ARO) under grant number W911NF-21-1-0198. Ming Yan is supported by NSF grant DMS-2012439 and Facebook Faculty Research Award (Systems for ML).


\bibliography{section/ref}
\bibliographystyle{icml2021}

\appendix
\onecolumn

\begin{center}
\large{\bf{Appendix for Elastic Graph Neural Networks}}
\end{center}

\section{Data Statistics}
\label{sec:data}

The data statistics for the benchmark datasets used in Section~\ref{sec:benchmark} are summarized in Table~\ref{tab:dataset}. The data statistics for the adversarially attacked graph used in Section~\ref{sec:attack} are summarized in Table~\ref{tab:attack_dataset}.

\begin{table*}[!ht]
\caption{Statistics of benchmark datasets.} 
\vskip 1em
% \resizebox{2\columnwidth}{!}{
\centering
\resizebox{\columnwidth}{!}{
\begin{tabular}{lcccccccc}
\toprule
\textbf{Dataset} &\textbf{Classes} &\textbf{Nodes} &\textbf{Edges} &\textbf{Features} &\textbf{Training Nodes} &\textbf{Validation Nodes} &\textbf{Test Nodes} \\
\midrule
Cora &7 &2708 &5278 &1433 &20 per class &500 &1000 \\
CiteSeer &6 &3327 &4552 &3703 &20 per class &500 &1000 \\
PubMed &3 &19717 &44324 &500 &20 per class &500 &1000 \\
Coauthor CS &15 &18333 &81894  &6805 &20 per class &30 per class &Rest nodes \\
Coauthor Physics &5 &34493 &247962 &8415 &20 per class &30 per class &Rest nodes\\
Amazon Computers &10 &13381 &245778& 767 &20 per class &30 per class &Rest nodes \\
Amazon Photo &8 &7487 &119043 & 745 &20 per class &30 per class &Rest nodes \\
\bottomrule
\end{tabular}
}
\label{tab:dataset}
\end{table*}

\begin{table}[h!]
\centering
% \hspace{0.1in}
\caption{Dataset Statistics for adversarially attacked graph.}
\vskip 1em
\resizebox{0.45\columnwidth}{!}{
\begin{tabular}{c|cccc}
\toprule
         & \textbf{$\mathbf{N_{LCC}}$} & \textbf{$\mathbf{E_{LCC}}$}  & \textbf{Classes} & \textbf{Features} \\ \midrule
Cora     & 2,485 & 5,069  & 7       & 1,433     \\ 
CiteSeer & 2,110 & 3,668  & 6       & 3,703     \\ 
Polblogs & 1,222 & 16,714 & 2       & /        \\ 
PubMed	 &19,717 & 44,338 & 3       & 500 \\ \bottomrule
\end{tabular}
}
\vspace{-1em}
\label{tab:attack_dataset}
\end{table}

\section{Convergence Guarantee}

\setcounter{theorem}{0}
We provide Theorem~\ref{thm:converge} to show the convergence guarantee of the proposed elastic messsage passing scheme and the practical guidance for parameter settings in EMP.

\begin{theorem}[Convergence of EMP]
\label{thm:converge}
Under the stepsize setting $\gamma < \frac{2}{1+\lambda_2\|\tL\|_2}$ and $\beta \leq \frac{4}{3\gamma \|\tDelta \tDelta^{\top}\|_2}$, the elastic message passing scheme (EMP) in Figure~\ref{fig:mp} converges to the optimal solution of the elastic graph signal estimator defined in~\eqref{eq:problem_normalized_l1} (Option I) or ~\eqref{eq:problem_normalized_l21} (Option II). 
It is sufficient to choose any $\gamma < \frac{2}{1+2\lambda_2}$ and $\beta \leq \frac{2}{3\gamma}$ since $\|\tL\|_2 = \|\tDelta^{\top} \tDelta\|_2  = \|\tDelta \tDelta^{\top}\|_2  \leq 2$.
\end{theorem}

\begin{proof}
\label{supp:convergence}
We first consider the general problem
\begin{align}
    \min_\vF f(\vF) + g(\vB\vF)
    \label{eq:general_prob}
\end{align}
where $f$ and $g$ are convex functions and $\vB$ is a bounded linear operator. 
It is proved in~\citep{loris2011generalization, chen2013primal} 
that the iterations in~\eqref{eq:opt_step1}--\eqref{eq:opt_step3} guarantee the convergence of $\vF^k$ to the optimal solution of the minimization problem~\eqref{eq:general_prob} if the parameters satisfy $\gamma < \frac{2}{L}$ and $\beta \leq \frac{1}{\gamma \lambda_{\max}(\vB\vB^{\top})}$, where $L$ is the Lipschitz constant of $\nabla f(\vF)$. These conditions are further relaxed to $\gamma < \frac{2}{L}$  and $\beta \leq \frac{4}{3\gamma \lambda_{\max}(\vB\vB^{\top})}$ in~\citep{li2017primal}.

For the specific problems defined in~\eqref{eq:problem_normalized_l1} and~\eqref{eq:problem_normalized_l21}, the two function components $f$ and $g$ are both convex, and the linear operator $\Delta$ is bounded. The Lipschitz constant of $\nabla f(\vF)$ can be computed by the largest eigenvalue of the Hessian matrix of $f(\vF)$: $$L=\lambda_{\max} (\nabla^2 f(\vF)) = \lambda_{\max} (\vI+\lambda_2 \tL) = 1+\lambda_2 \|\tL\|_2.$$ 
Therefore, the elastic message passing scheme derived from iterations~\eqref{eq:opt_step1}--\eqref{eq:opt_step3} is guaranteed to converge to the optimal solution of problem~\eqref{eq:problem_normalized_l1} (Option I) or problem~\eqref{eq:problem_normalized_l21} (Option II)  if the stepsizes satisfy $\gamma < \frac{2}{1+\lambda_2 \|\tL\|_2}$ and $\beta \leq \frac{4}{3\gamma \|\tDelta \tDelta^\top\|_2}$. 

Let $\tDelta=\vU\Sigma\vV^{\top}$ be the singular value decomposition of $\tDelta$, and we derive 
$$
\|\tDelta\tDelta^{\top}\|_2 = \|\vU\Sigma\vV^{\top} \vV \Sigma\vU^{\top}\|_2=\|\vU\Sigma^2\vU^{\top}\|_2=\|\vV\Sigma^2\vV^{\top}\|_2=\|\vV\Sigma\vU^{\top} \vU \Sigma\vV^{\top}\|_2 = \|\tDelta^{\top}\tDelta\|_2
.$$
The equivalence $\tL = {\tDelta}^{\top} \tDelta$ in~\eqref{eq:normalized_inc} further gives 
$$\|\tL\|_2=\|\tDelta^{\top}\tDelta\|_2=\|\tDelta\tDelta^{\top}\|_2.$$
Since $\|\tL\|_2 \leq 2$~\citep{chung1997spectral}, we have
$\frac{2}{1+2\lambda_2} \leq \frac{2}{1+\lambda_2\|\tL\|_2}$ and $\frac{2}{3\gamma} \leq \frac{4}{3\gamma \|\tDelta \tDelta^\top\|_2}$. Therefore, $\gamma < \frac{2}{1+2\lambda_2}$
$\beta \leq \frac{2}{3\gamma}$ are sufficient for the convergence of EMP.

\end{proof}

\end{document}